\documentclass{article}

\usepackage[utf8]{inputenc}
\usepackage[T1]{fontenc}
\usepackage{hyperref}
\usepackage{url}
\usepackage{booktabs}
\usepackage{amsfonts}
\usepackage{nicefrac}
\usepackage{microtype}
\usepackage{xcolor}
\usepackage{float}
\usepackage{graphicx}
\usepackage{bm}
\usepackage[margin=1.03in]{geometry}
\usepackage{algorithm}
\usepackage{algorithmic}
\usepackage{enumerate}
\usepackage{amsmath}
\usepackage{amssymb}
\usepackage{mathtools}
\usepackage{amsthm}
\usepackage{mathrsfs}

\newtheorem{theorem}{Theorem}

\newcommand{\A}{\mathcal{A}}
\newcommand{\E}{\mathbb{E}}
\renewcommand{\Pr}{\mathbb{P}}
\newcommand{\wh}{\widehat}
\newcommand{\R}{\mathbb{R}}
\newcommand{\cE}{\mathcal{E}}
\newcommand{\cF}{\mathcal{F}}

\hypersetup{
  colorlinks=true,
  linkcolor=blue!55!black,
  citecolor=blue!55!black,
  urlcolor=blue!65!black,
  pdftitle={Dimension-Adaptive Batched Lipschitz Narrowing Without Knowing the Zooming Dimension}
}

\title{Dimension-Adaptive Batched Lipschitz Narrowing\\
Without Knowing the Zooming Dimension}
\author{Yasong Feng}
\date{}

\begin{document}

\maketitle
\vspace{-1.2cm}

\begin{abstract}
The Appropriately Combined Edge-length (ACE) sequence in A-BLiN depends
on the zooming dimension $d_z$.  This note removes that dependence.  The
next edge length is selected from the number of cubes that survive the
preceding elimination.  The resulting Count-Adaptive BLiN algorithm does
not use $d_z$ or the zooming constant $C_z$, yet it attains
$\widetilde{\mathcal O}_d(T^{(d_z+1)/(d_z+2)})$ regret with
$\mathcal O_d(\log\log T)$ batches.  Together with the adaptive-grid lower
bound in Theorem~10 of the original paper, the optimal batch complexity
remains $\Theta_d(\log\log T)$ when $d_z$ is unknown.
\end{abstract}

\section{Setup and relation to the original results}

We retain the notation and model of \cite{feng2022blin}.  The arm space is
$\A=[0,1]^d$ equipped with $\|\cdot\|_\infty$, the mean reward
$\mu:\A\to\R$ is $1$-Lipschitz, and
$\Delta_x=\mu^\star-\mu(x)$.  Recall
\[
  S(r)=\{x\in\A:\Delta_x\le r\},
  \qquad
  N_r=\mathcal N\!\left(S(16r),\frac r2\right),
\]
and
\[
  N_r\le C_zr^{-d_z},\qquad 0<r<1.
\]
As in Algorithm~1 of \cite{feng2022blin}, $\A_m$ denotes the active
standard cubes in batch $m$, all having edge length $r_m$;
$\A_m^+$ denotes the cubes surviving the elimination in that batch; and
\[
  \wh\mu_m(C)=\frac1{n_m}\sum_{i=1}^{n_m}y_{C,i},
  \qquad
  \wh\mu_m^{\max}=\max_{C\in\A_m}\wh\mu_m(C).
\]

Theorem~4 of \cite{feng2022blin} shows that D-BLiN does not require
$d_z$ and attains the optimal regret exponent using $\mathcal O(\log T)$
batches.  Definition~3 and Theorem~5 use the $d_z$-dependent ACE sequence
to reduce this to $\mathcal O(\log\log T)$ batches.  Theorem~6 handles the
rounding of the ACE sequence.  The algorithm below instead makes the edge
length data-dependent and dyadic from the outset.

Let
\[
  H:=16\log T,
  \qquad
  a_0:=\frac{d+1}{d+2},
\]
and, for $T\ge3$, let
\begin{equation}
  B_0:=\max\left\{1,
  \left\lceil
    \frac{\log\log T}{\log\frac{d+2}{d+1}}
  \right\rceil\right\}.
  \label{eq:def-B0}
\end{equation}
For $s\in(0,1]$, define
\[
  \operatorname{dyad}(s):=2^{-\lceil\log_2(1/s)\rceil},
\]
the largest dyadic number not exceeding $s$.

For a collection $\mathcal C$ of cubes and a time $t$, the instruction
$\operatorname{Cleanup}(\mathcal C,t)$ means: play arbitrary arms in
$\bigcup_{C\in\mathcal C}C$ for the remaining $T-t$ rounds, without an
intermediate feedback round, and then terminate.

\begin{algorithm}[H]
\caption{Count-Adaptive Batched Lipschitz Narrowing (CA-BLiN)}
\label{alg:ca-blin}
\begin{algorithmic}[1]
\STATE \textbf{Input:} arm set $\A=[0,1]^d$ and time horizon $T$.
\STATE Initialize $r_0=1$, $\A_0^+=\{\A\}$, and $t_0=0$.
\FOR{$m=1,2,\ldots,B_0$}
  \STATE Set
  $K_{m-1}=|\A_{m-1}^+|$ and
  $\displaystyle\gamma_{m-1}=\frac{HK_{m-1}}{Tr_{m-1}^2}$.
  \IF{$\gamma_{m-1}\ge1$}
    \STATE $\operatorname{Cleanup}(\A_{m-1}^+,t_{m-1})$.
    \RETURN
  \ENDIF
  \STATE Set
  $\displaystyle\bar r_m=r_{m-1}\gamma_{m-1}^{1/[2(d+2)]}$,
  $r_m=\operatorname{dyad}(\bar r_m)$, and
  $\displaystyle n_m=\left\lceil\frac{H}{r_m^2}\right\rceil$.
  \STATE Partition every $C\in\A_{m-1}^+$ into
  $(r_{m-1}/r_m)^d$ standard cubes of edge length $r_m$;
  denote their collection by $\A_m$.
  \STATE Set $L_m=|\A_m|n_m$.
  \IF{$t_{m-1}+L_m>T$}
    \STATE $\operatorname{Cleanup}(\A_{m-1}^+,t_{m-1})$.
    \RETURN
  \ENDIF
  \STATE Play every $C\in\A_m$ exactly $n_m$ times and collect the
  rewards at $t_m=t_{m-1}+L_m$.
  \STATE Compute $\wh\mu_m(C)$ and
  $\wh\mu_m^{\max}=\max_{C\in\A_m}\wh\mu_m(C)$.
  \STATE Eliminate $C$ if
  $\wh\mu_m^{\max}-\wh\mu_m(C)>4r_m$; let $\A_m^+$ be the surviving
  cubes.
\ENDFOR
\STATE $\operatorname{Cleanup}(\A_{B_0}^+,t_{B_0})$.
\end{algorithmic}
\end{algorithm}

Neither $d_z$ nor $C_z$ is an input to Algorithm~\ref{alg:ca-blin}.
Moreover, $r_m$, $\A_m$, and $t_m$ are measurable with respect to the
observations available at time $t_{m-1}$, so the resulting batch grid is
adaptive in exactly the sense used in \cite{feng2022blin}.

\section{Regret and batch complexity}

\begin{theorem}[Dimension-adaptive regret and batch complexity]
\label{thm:ca-blin}
Assume the model of \cite{feng2022blin}, with $T\ge3$ and fixed ambient
dimension $d$.  Let $d_z$ and $C_z$ be the zooming dimension and zooming
constant of the instance, and define the analysis-only critical radius
\begin{equation}
  \rho:=\left(\frac{HC_z}{T}\right)^{1/(d_z+2)}.
  \label{eq:def-rho}
\end{equation}
Then, with probability at least $1-2T^{-7}$,
Algorithm~\ref{alg:ca-blin} satisfies
\begin{equation}
  R(T)\le 2^{d+6}(B_0+1)T\rho.
  \label{eq:main-bound-rho}
\end{equation}
Equivalently,
\begin{equation}
  R(T)\le
  2^{d+6}(B_0+1)
  C_z^{1/(d_z+2)}
  T^{\frac{d_z+1}{d_z+2}}
  (16\log T)^{1/(d_z+2)}.
  \label{eq:main-bound}
\end{equation}
The total number of batches is at most $B_0+1=\mathcal O_d(\log\log T)$.
Consequently, a single policy that does not know $d_z$ or $C_z$ attains
the optimal $T$-regret exponent of Theorem~5 of \cite{feng2022blin} with
the optimal batch-complexity order.
\end{theorem}

\begin{proof}
To prove the theorem, we first show that the concentration and
elimination properties in Lemmas~1--3 of \cite{feng2022blin} continue to
hold for the data-dependent edge lengths.  Based on these results, we
bound the regret of each completed refinement batch and the Cleanup step.
Finally, we show that $B_0$ refinement batches are sufficient.

First suppose that $\rho\ge1$.  Since
$\operatorname{diam}_{\infty}([0,1]^d)=1$, the Lipschitzness of $\mu$
gives $\Delta_x\le1$ for every arm $x$.  Hence,
$R(T)\le T\le T\rho$, and the regret bound follows directly.  The batch
bound follows from the construction of the algorithm.  In the following,
we assume $\rho<1$.

\paragraph{Adaptive-scale concentration.}
For each executed $C\in\A_m$, define
\[
  \bar\mu_m(C):=\frac1{n_m}\sum_{i=1}^{n_m}\mu(x_{C,i}).
\]
Fix an executed cube $C\in\A_m$.  Conditional on the history
$\cF_{t_{m-1}}$ available at the beginning of batch $m$ (and on any fresh
private randomization used to choose the sampling locations), $r_m$,
$\A_m$, $n_m$, and all sampling locations in this batch are fixed.
Therefore, the same Gaussian tail inequality as in Lemma~1 of
\cite{feng2022blin} gives
\begin{align}
 &\Pr\!\left(
   |\wh\mu_m(C)-\bar\mu_m(C)|>r_m
   \mid\cF_{t_{m-1}}
 \right)                                                     \notag\\
 &\hspace{3cm}\le
 2\exp\!\left(-\frac{n_mr_m^2}{2}\right)
 \le2\exp(-H/2)=2T^{-8}.                    \label{eq:cond-tail}
\end{align}
On the other hand, by the Lipschitzness of $\mu$, for every $x\in C$,
\[
  |\bar\mu_m(C)-\mu(x)|\le r_m.
\]
Let
\begin{equation}
  \cE:=\left\{
  |\wh\mu_m(C)-\mu(x)|\le2r_m
  \text{ for every executed }(m,C)\text{ and every }x\in C
  \right\}.
  \label{eq:event-E}
\end{equation}
Every estimated cube is played at least once.  Therefore,
\[
  \sum_m|\A_m|
  \le\sum_m|\A_m|n_m\le T.
\]
From here, applying the above conditional probability bound whenever a
cube is executed and then taking a union bound over at most $T$ executed
cubes gives
\begin{equation}
  \Pr(\cE^c)
  \le2T^{-8}\E\!\left[\sum_m|\A_m|\right]
  \le2T^{-7}.
  \label{eq:event-prob}
\end{equation}
This proves that the conclusion of Lemma~1 of \cite{feng2022blin} remains
valid for the data-dependent edge lengths used by CA-BLiN.

\paragraph{Elimination invariants.}
In the following, we work under event $\cE$.  We first show that an
optimal arm survives all eliminations.  Let $C_m^\star\in\A_m$ be the
cube containing an optimal arm $x^\star$.  Under event $\cE$, for any
cube $C\in\A_m$ and any $x\in C$, we have
\[
  \wh\mu_m(C)-\wh\mu_m(C_m^\star)
  \le \mu(x)+2r_m-\mu(x^\star)+2r_m
  \le4r_m.
\]
Then from the strict elimination rule, $C_m^\star$ is not eliminated.
This is the same argument as in Lemma~2 of \cite{feng2022blin}.

Based on this result, we show that the cubes surviving elimination are
of high reward.  Fix $C\in\A_m^+$ and $x\in C$.  Since $C$ is not
eliminated, $\wh\mu_m^{\max}-\wh\mu_m(C)\le4r_m$.  Under event $\cE$,
it holds that
\begin{align}
  \Delta_x
  &=\mu^\star-\mu(x)                                      \notag\\
  &\le
  \wh\mu_m(C_m^\star)+2r_m-\wh\mu_m(C)+2r_m             \notag\\
  &\le
  \wh\mu_m^{\max}-\wh\mu_m(C)+4r_m
  \le8r_m.                                  \label{eq:survivor-gap}
\end{align}
Hence, we conclude that $\Delta_x\le8r_m$.  This is the survivor form of
Lemma~3 of \cite{feng2022blin}.  In particular, every survivor cube is
contained in $S(8r_m)$.  Moreover, the centers of the survivor cubes form
an $r_m/2$-packing of a subset of
$S(8r_m)\subseteq S(16r_m)$.  Therefore, by the same argument used to
obtain inequality~(4) in the proof of Theorem~5 of \cite{feng2022blin},
\begin{equation}
  K_m:=|\A_m^+|\le N_{r_m}\le C_z r_m^{-d_z}.
  \label{eq:survivor-count}
\end{equation}
The same two bounds also hold for the virtual initial state $m=0$.
Indeed, $\operatorname{diam}_{\infty}([0,1]^d)=1$ gives
$\Delta_x\le1\le8r_0$ for every $x\in\A_0^+$.  Moreover, $K_0=1\le C_z$,
because $N_r\ge1$ for every $0<r<1$, and the inequality
$N_r\le C_zr^{-d_z}$ in Definition~2 of \cite{feng2022blin}, together
with $r\uparrow1$, implies $C_z\ge1$.  Thus, for every state $m\ge0$
reached by the algorithm,
\[
  \Delta_x\le8r_m
  \quad\left(x\in\bigcup_{C\in\A_m^+}C\right),
  \qquad
  K_m\le C_zr_m^{-d_z}.
\]
Finally, every arm played in batch $m+1$ belongs to a cube in
$\A_m^+$.  Hence,
\begin{equation}
  R_{m+1}\le8r_mL_{m+1}.
  \label{eq:batch-parent-gap}
\end{equation}
Note that this bound involves the parent edge length $r_m$, rather than
the new edge length $r_{m+1}$.

\paragraph{Length and regret of one refinement batch.}
We are now ready to bound the regret of one completed refinement batch.
Fix a state with survivor edge length $r_m$, survivor count $K_m$, and
$\gamma_m<1$.  By the definition of $\bar r_{m+1}$,
\begin{equation}
  \bar r_{m+1}^{d+2}
  =r_m^{d+2}\sqrt{\gamma_m}
  =r_m^{d+1}\sqrt{\frac{HK_m}{T}}.
  \label{eq:ideal-balance}
\end{equation}
Moreover, $\bar r_{m+1}/2<r_{m+1}\le\bar r_{m+1}$.  Since
$H/r_{m+1}^2\ge1$, the definition of $n_{m+1}$ gives
$n_{m+1}\le2H/r_{m+1}^2$.  Therefore, the length of batch $m+1$
satisfies
\begin{align}
  L_{m+1}
  &=K_m\left(\frac{r_m}{r_{m+1}}\right)^dn_{m+1}          \notag\\
  &\le\frac{2HK_mr_m^d}{r_{m+1}^{d+2}}                   \notag\\
  &<\frac{2^{d+3}HK_mr_m^d}{\bar r_{m+1}^{d+2}}          \notag\\
  &=2^{d+3}\frac{\sqrt{THK_m}}{r_m}
   =2^{d+3}T\sqrt{\gamma_m}.                \label{eq:batch-length}
\end{align}
Combining \eqref{eq:batch-parent-gap} and \eqref{eq:batch-length} gives
\begin{equation}
  R_{m+1}<2^{d+6}\sqrt{THK_m}.
  \label{eq:one-batch-regret}
\end{equation}

Equality $HC_z=T\rho^{d_z+2}$, together with
inequality~\eqref{eq:survivor-count}, gives
\begin{equation}
  \sqrt{THK_m}
  \le T\rho\left(\frac{\rho}{r_m}\right)^{d_z/2}.
  \label{eq:oracle-batch-bound}
\end{equation}
If $r_m\ge\rho$, the right-hand side is at most $T\rho$.  If
$r_m<\rho$, then $\gamma_m<1$ gives
\[
  \sqrt{THK_m}=Tr_m\sqrt{\gamma_m}<Tr_m<T\rho.
\]
Consequently, whether $r_m\ge\rho$ or $r_m<\rho$, every completed
refinement batch satisfies
\begin{equation}
  R_{m+1}\le2^{d+6}T\rho.
  \label{eq:uniform-batch-regret}
\end{equation}

\paragraph{The two early-cleanup cases.}
We next bound the regret in the two cases where the algorithm enters the
Cleanup step before reaching the hard cap.

First suppose that $\gamma_m\ge1$.  By
inequality~\eqref{eq:survivor-count} and the definition of $\rho$,
\[
  1\le\gamma_m
  \le\frac{HC_z}{Tr_m^{d_z+2}}
  =\left(\frac{\rho}{r_m}\right)^{d_z+2},
\]
and hence $r_m\le\rho$.  Inequality~\eqref{eq:survivor-gap} gives
\begin{equation}
  R_{\mathrm{cleanup}}
  \le8r_m(T-t_m)\le8T\rho.
  \label{eq:cleanup-gamma}
\end{equation}

Next suppose that $\gamma_m<1$, but the planned batch does not fit in the
remaining horizon, so that $L_{m+1}>T-t_m$.  The algorithm then cleans up
in the old survivor region.  From \eqref{eq:survivor-gap} and
\eqref{eq:batch-length}, we have
\begin{align}
  R_{\mathrm{cleanup}}
  &\le8r_m(T-t_m)
   <8r_mL_{m+1}                                             \notag\\
  &<2^{d+6}\sqrt{THK_m}                                    \notag\\
  &\le2^{d+6}T\rho.                         \label{eq:cleanup-fit}
\end{align}
The last inequality follows from \eqref{eq:oracle-batch-bound}
when $r_m\ge\rho$, and from
$\sqrt{THK_m}=Tr_m\sqrt{\gamma_m}<T\rho$ when $r_m<\rho$.

\paragraph{Scale contraction and the hard cap.}
It remains to show that the Cleanup step after $B_0$ completed
refinements also has small regret.  Suppose that the refinement from
$r_m$ to $r_{m+1}$ is completed.  When $r_m\ge\rho$,
inequality~\eqref{eq:survivor-count} and
equality~\eqref{eq:def-rho} give
\[
  \gamma_m\le\left(\frac{\rho}{r_m}\right)^{d_z+2}.
\]
Consequently,
\begin{equation}
  r_{m+1}
  \le\bar r_{m+1}
  \le\rho^{1-a_z}r_m^{a_z},
  \qquad
  a_z:=1-\frac{d_z+2}{2(d+2)}
  \le a_0.
  \label{eq:scale-contraction}
\end{equation}
If $r_j<\rho$ for some $j\le B_0$, then every subsequent completed
refinement satisfies $r_{k+1}<r_k<\rho$ for $k\ge j$, and the desired
conclusion already holds.  Otherwise, $r_m\ge\rho$ before every completed
refinement.  Iterating \eqref{eq:scale-contraction} and using $r_0=1$
gives
\begin{equation}
  \log\frac{r_{B_0}}{\rho}
  \le a_0^{B_0}\log\frac1\rho.
  \label{eq:iterate-contraction}
\end{equation}
By the definition of $B_0$, we have
$a_0^{B_0}\le1/\log T$.  Moreover, $H\ge1$ and Definition~2 of
\cite{feng2022blin} gives $C_z\ge1$.  Hence,
\[
  \log\frac1\rho
  =\frac{\log(T/(HC_z))}{d_z+2}
  \le\frac{\log T}{d_z+2}.
\]
Consequently,
\begin{equation}
  r_{B_0}\le e^{1/(d_z+2)}\rho\le\sqrt e\,\rho.
  \label{eq:final-radius}
\end{equation}
Thus the regret of the hard-cap Cleanup batch is at most
\begin{equation}
  8T r_{B_0}\le8\sqrt e\,T\rho<2^{d+6}T\rho.
  \label{eq:hard-cap-cleanup}
\end{equation}

\paragraph{Summation.}
There are at most $B_0$ completed refinement batches.  By
\eqref{eq:uniform-batch-regret}, their total regret is at most
$2^{d+6}B_0T\rho$.  After these batches, the algorithm uses at most one
Cleanup batch.  Its regret is bounded by \eqref{eq:cleanup-gamma},
\eqref{eq:cleanup-fit}, or \eqref{eq:hard-cap-cleanup}, according to the
stopping rule.  Therefore,
\[
  R(T)\le2^{d+6}(B_0+1)T\rho.
\]
This proves \eqref{eq:main-bound-rho}.  Substituting the definition of
$\rho$ into this inequality gives \eqref{eq:main-bound}.  The same
counting argument shows that the total number of batches is at most
$B_0+1$.  Since $R(T)\le T$, \eqref{eq:event-prob} also gives
$\E[R(T)]\le2^{d+6}(B_0+1)T\rho+2T^{-6}$.  Together with Theorem~10
(Theorem~3 in the introduction) and Corollary~1 of
\cite{feng2022blin}, this proves the optimal
$\Theta_d(\log\log T)$ batch-complexity order and finishes the proof.
\end{proof}

\end{document}